\documentclass[11pt, onecolumn]{article}
\usepackage[top=1in, bottom=1in, left=1.25in, right=1.25in]{geometry}

\usepackage{natbib}
\setcitestyle{authoryear,open={(},close={)}}

\usepackage{bm}
\usepackage[cmex10]{amsmath}

\usepackage{graphicx}
\usepackage{amsfonts}
\usepackage[cmex10]{amsmath}
\usepackage{amssymb}
\usepackage{amsthm}
\usepackage{mdwmath}
\usepackage{algorithm}
\usepackage{algorithmic}
\usepackage{subfigure}
\usepackage{subeqnarray}
\usepackage{cases}
\usepackage{tabularx}
\usepackage{multirow}
\usepackage{enumerate}
\usepackage{subfigure}
\usepackage{url}
\usepackage{flushend}
\usepackage{tikz}
\usepackage{array}
\usepackage{booktabs}
\usepackage{color,soul}

\theoremstyle{definition}

\theoremstyle{definition}

\theoremstyle{definition}
\newtheorem{lemma}{Lemma}

\theoremstyle{definition}
\newtheorem{theorem}{Theorem}

\theoremstyle{definition}

\theoremstyle{remark}

\newcommand{\va}{{\bm{a}}}
\newcommand{\vb}{{\bm{b}}}

\newcommand{\ve}{{\bm{e}}}

\newcommand{\vu}{{\bm{u}}}
\newcommand{\vv}{{\bm{v}}}
\newcommand{\vw}{{\bm{w}}}
\newcommand{\vx}{{\bm{x}}}
\newcommand{\vy}{{\bm{y}}}
\newcommand{\vz}{{\bm{z}}}

\newcommand{\vA}{{\bm{A}}}

\newcommand{\vE}{{\bm{E}}}

\newcommand{\vI}{{\bm{I}}}

\newcommand{\vM}{{\bm{M}}}

\newcommand{\vP}{{\bm{P}}}

\newcommand{\vR}{{\bm{R}}}

\newcommand{\vU}{{\bm{U}}}

\newcommand{\vW}{{\bm{W}}}
\newcommand{\vX}{{\bm{X}}}

\graphicspath{{../}}

\newcommand{\me}{\mathrm{e}}
\newcommand{\Bl}{{\Big |}}
\newcommand{\pr}{{\mathbb{P}}}
\newcommand{\ex}{{\mathbb{E}}}

\begin{document}

\title{Theory of Spectral Method for Union of Subspaces-Based Random Geometry Graph}

\author{Gen~Li and~Yuantao~Gu%
\thanks{ 
The authors are with Department of Electronic Engineering, Tsinghua University, Beijing 100084, China. 
The corresponding author of this paper is Y. Gu (gyt@tsinghua.edu.cn).}
}
\date{Manuscript submitted July 23, 2019.}

\maketitle

\begin{abstract}
Spectral Method is a commonly used scheme to cluster data points lying close to Union of Subspaces by first constructing a Random Geometry Graph, called Subspace Clustering. 
This paper establishes a theory to analyze this method.
Based on this theory, we demonstrate the efficiency of Subspace Clustering in fairly broad conditions.
The insights and analysis techniques developed in this paper might also have implications for other random graph problems.
Numerical experiments demonstrate the effectiveness of our theoretical study.

\end{abstract}

\section{Introduction}

\subsection{Motivation}

Union of Subspaces (UoS) model serves as an important model in statistical machine learning. 
Briefly speaking, UoS models those high-dimensional data, encountered in many real-world problems, 
which lie close to low-dimensional subspaces corresponding to several classes to which the data belong, 
such as hand-written digits~\citep{hastie1998metrics}, face images~\citep{basri2003lambertian},  DNA microarray data~\citep{parvaresh2008recovering}, and hyper-spectral images~\citep{chen2011hyperspectral}, to name just a few. 
A fundamental task in processing data points in UoS is to cluster these data points, which is known as Subspace Clustering (SC). 
Applications of SC has spanned all over science and engineering, 
including motion segmentation~\citep{costeira1998multibody, kanatani2001motion}, face recognition~\citep{wright2008robust}, and classification of diseases~\citep{mcwilliams2014subspace} and so on.
We refer the reader to the tutorial paper~\citep{vidal2011subspace} for a review of the development of SC.

Considering the wide applications of SC, 
numerous algorithms have been developed for subspace clustering~\citep{tipping1999mixtures, tseng2000nearest, vidal2005generalized, yan2006general, elhamifar2009sparse, peng2018structured, meng2018general}.
Arguably, a series of two-step algorithms, referring to Sparse Subspace Clustering (SSC) and its variants~\citep{elhamifar2009sparse, liu2012robust, Dyer2013Greedy, Heckel2015robust, chen2017active}, are the most popular and efficient methods for solving SC,
which first construct a random graph (or an adjacent matrix equivalently), named as Union of Subspaces-based Random Geometry Graph (UoS-RGG), depending on the relative position among data points, 
and then apply the spectral method~\citep{ng2002spectral, von2007tutorial} to obtain the clustering result.

In spite of all these algorithms that practically work well for many applications, 
theoretical guarantees are lacked for the accuracy of clustering of any SC algorithm.
We note that although novel and often efficient subspace clustering techniques emerge all the time, 
establishing rigorous theory for such techniques is quite difficult and does not exist as of now. 
The fundamental difficulty in the analysis of SC algorithms 
may be the change of view required in treating UoS-RGG (or general Random Geometry Graph, RGG), 
which has non-independent edges, 
in contrast with the traditional approach to analyzing clustering algorithms 
via Stochastic Block Model (SBM) which assumes independent edges. 
Section~\ref{sec:related} offers a detailed discussion of this difficulty, 
as well as a survey of the existing attempts in theoretical aspects. 
We therefore propose the critical question that this paper aims to explore: 
\begin{itemize}
\item Why does SC work, or more precisely, why does spectral method work for RGG or UoS-RGG?
\end{itemize}

This paper focuses on the analysis on the spectral method for UoS-RGG.
We consider a naive and prototypical SC algorithm (Algorithm~\ref{alg:TIP-SC}) here,
and prove this algorithm, though oversimplified, can still deliver an almost correct clustering result
even when the subspaces are quite close to each other 
and when the number of samples is far less than the subspace dimension (see Theorem~\ref{thm:result}). 
To the best of our knowledge, this is the first ever theory established to analyze the clustering error of 
SC algorithm.
It not only constitutes the first theoretical guarantee for accuracy of subspace clustering, 
but also provides the interesting insight that the widely-conjectured oversampling requirement 
for subspace clustering is redundant, and that subspace clustering is quite robust 
in existence of closely aligned subspaces. 
We also verify our results by numerical experiments in Section~\ref{sec:exp}.
Although our theoretical results is proved only for the simplified algorithm we choose, 
it should be quite convincing that more carefully-designed SC algorithms would give even better performance than what we guarantee here, 
and our proof could serve as a prototype to the analysis of these algorithms. 

\subsection{Related Works and Challenges}
\label{sec:related}

We now briefly review the literature on the adjacent matrix and spectral method 
and discuss their shortcomings.
Since this paper mainly deals with theory, we shall focus on theoretical aspects of existing results. 

\subsubsection{Analysis of Random Graphs for UoS}
To analyze the random graphs associated to UoS model in an abstract setting 
without referring to any specific algorithms, 
most researches focus on the Subspace Detection Property~\citep[SDP,][]{soltanolkotabi2012geometric, liu2012robust, soltanolkotabi2014robust}, 
a property which indicates that there are no edge connections between the data points in different subspaces,
but are many connections between the data points in the same subspace.
Under some technical conditions on the parameters of SC,
the random graphs constructed by a variety of SC algorithms have been proved to enjoy SDP.
Readers may consult Section 3 in \citet{soltanolkotabi2014robust} for details.

There are, however, two main deficiencies of SDP which render SDP hard to use in further analysis.
The first one is that SDP does not imply a correct clustering result.
Actually, one can easily construct a counter-example where SDP holds but the clustering result is unsatisfying.
The second one is that SDP requires too restrictive conditions on affinity between subspaces and sampling rate to hold. 
These conditions are provably unnecessary, as will be demonstrated in Section~\ref{sec:result} of this paper.

\subsubsection{Analysis of Spectral Method for Random Graphs}
Compared with SDP, a more concrete approach to analyze SC algorithm is 
to investigate the performance of spectral method on random graphs associated to UoS model. 
To this end, analysis of spectral method 
for general random graphs (not necessarily associated to UoS model) is relevant. 
Note that the spectral method is explored deeply in the literature of community detection,
which is an important problem in statistics, computer vision, and image processing~\citep{abbe2017community}. 
Stochastic Block Model (SBM) is a widely used theoretical model in this field, 
which we briefly introduce as follows.
For simplicity, we consider the two-block case, 
where the vertices of random graph are divided into two ``blocks'', 
i.e. sets of vertices that ought to be closely-related, each of size of $N/2$.
Then each edge of random graph is independently generated from the following distribution: 
for $p>q>0$, vertices $\vx_i$ and $\vx_j$ are connected with probability $p$ if $\vx_i, \vx_j$ belong to the same block, 
and with probability $q$ if they belong to different blocks.
Given an instance of this graph, we would like to identify the two blocks. 
Recently, a series of theoretical works are devoted to analyze the performance of spectral method on this problem in different settings~\citep{coja2010graph, vu2014simple, chin2015stochastic, abbe2017entrywise}, and extensions~\citep{sankararaman2018community}.

As far as we know, 
all existing results make essential use of the independence of different edges, 
which is unfortunately not the case in SC algorithms. 
In fact, it is a generic and natural phenomenon in RGG that 
when $\vx_i, \vx_j$ and $\vx_i, \vx_k$ are connected, 
the probability that $\vx_j, \vx_k$ are connected will be higher,
hence the independence assumption does not hold for RGG.

With this fundamental gap in mind, 
it is crucial to develop a theory for RGG to provide a rigorous theoretical guarantee for SC algorithms.

\section{Preliminaries and Problem Formulation}
\label{sec:notation}

The generative model for data points in UoS we adapt in this paper is the semi-random model introduced in \citet{soltanolkotabi2012geometric},
which assumes that the subspaces are fixed with points distributed uniformly at random on each subspace.
This is arguably the simplest model providing a good starting point for a theoretical investigation. 
We assume the data consists of two clusters, corresponding to two fixed subspaces\footnote{It should be noticed that the number of subspaces is by no means crucial to the analysis. 
The results in this paper can be generalized to more subspaces easily.}
 $S_1, S_2$ in $\mathbb{R}^n$,
each with $N/2$ data points uniformly sampled from the unit spheres $\mathcal{S}_1^{d-1}$ and $\mathcal{S}_2^{d-1}$ respectively in $S_1$ and $S_2$.
Here $d$ is the subspace dimension and $n$ is the ambient dimension.
The goal of SC is to cluster the normalized data points $\{\vx_i\}_{1 \le i \le N}$.

Given the general description of SC, 
we turn our attention to a simple prototypical SC algorithm detailed in Algorithm~\ref{alg:TIP-SC}, 
which we call Thresholding Inner-Product Subspace Clustering (TIP-SC). 
Considering that the angle between the data points in the same subspaces would be smaller statistically,
we construct for some threshold $\tau \in (0, 1)$ the random graph 
by computing its adjacent matrix $\vA$, 
where $A_{ij} = 1$ if $i \ne j, |\langle \vx_i, \vx_j\rangle| \ge \tau$, and $A_{ij} = 0$ otherwise. 
The TIP-SC algorithm concludes with applying the spectral clustering method on $\vA$.

The main task of this paper is to prove this simple algorithm can achieve a high clustering accuracy under fairly general condition, 
which will be done in the next section.

\begin{algorithm}[t]
\caption{Thresholding inner-product subspace clustering (TIP-SC)}
\label{alg:TIP-SC}
\begin{algorithmic}[1]
  \REQUIRE{Normalized data set $\{\vx_i\}_{1 \le i \le N}$, threshold $\tau$.}
  \STATE{\textbf{Construct Adjacent Matrix $\vA$:} }
  \STATE{\quad $A_{ij} = 1$ if $i \ne j, |\langle \vx_i, \vx_j\rangle| \ge \tau$, or $A_{ij} = 0$ otherwise.}
  \STATE{\textbf{Apply Spectral Method on $\vA$:} 
  \STATE{\quad Calculate $\vW$, the eigenspace corresponding to the top two eigenvalues of $\vA$.}
  \STATE{\quad Use ${\rm sgn}(\vw)$ as clustering result, where $\vw$ is the vector in $\vW$ perpendicular to the projection of all-ones vector in $\vW$.}} 
\end{algorithmic}
\end{algorithm}

\paragraph{Notations.} 
Let $\vU_1, \vU_2$ denote the orthonormal bases for the subspaces $S_1, S_2$, respectively,
and $\lambda_1 \ge \ldots \ge \lambda_d \ge 0$ denote the singular values of $\vU_1^{\top}\vU_2$.
We also use $S$ and $S'$ to denote the subspaces to which $\vx_i$ does and doesn't belong, respectively.
Then $\vx_i = \vU\overline{\va}_i$ 
where $\vU$ denotes the orthonormal bases for $S$, 
$\va_i \overset{\mathrm{ind.}}{\sim} \mathcal{N}(\bm{0}, \frac{1}{d}\vI_d) \in \mathbb{R}^d$, 
and $\overline{\va}_i=\va_i/\|\va_i\|$ denotes its normalization.
We use $p, q$ to represent the probability that $A_{ij} = 1$ for $j \ne i, \vx_j \in S$ and $\vx_j \in S'$, respectively.
Conditioned on $\vx_i$, let $p_i$ denote the probability of $A_{ij} = 1$ for $j \ne i, \vx_j \in S$,
and $q_i$ denote the probability of $A_{ij} = 1$ for $j, \vx_j \in S'$.
Denote
\begin{align*}
  {\rm aff} &:= \sqrt{\frac{\sum_i \lambda_i^2}{d}},\\
  \kappa &:= 1 - {\rm aff}^2,\\
  \rho &:= \frac{N}{2d}. 
\end{align*}

Let $\vu, \vv \in \mathbb{R}^{N}$ with $u_i = \frac{1}{\sqrt{N}}$, 
and $v_i = \frac{1}{\sqrt{N}}$, if $\vx_i \in S_1$, and $v_i = -\frac{1}{\sqrt{N}}$, if $\vx_i \in S_2$,
then $\vv$ is the ground truth.
$\vW$ denotes the eigenspace corresponding to the top two eigenvalues of $\vA$,
and $\vw$ denotes the vector in $\vW$, 
which is perpendicular to the projection of $\vu$ in $\vW$.


\section{Error Rate of TIP-SC Algorithm}
\label{sec:result}

This section presents our main theoretical results concerning the performance of TIP-SC.
By the perturbation analysis of $\vA$ from $\ex \vA$, 
the success of spectral method for SBM has been proved in various statistical assumptions.
However, such analysis is insufficient to establish our result,
since for UoS-RGG, the independence condition doesn't hold,
which is the crux leading to the failure of the existing methods for analyzing spectral method on random graph.
As a substitute, we discover the conditional independence property for $\vA$,
based on which we prove that the clustering result of TIP-SC is almost correct under some mild condition on affinity and sampling rate,
which is explained in the following theorem.

\begin{theorem} \label{thm:result}
Choosing $\tau = O\left(\frac{1}{\sqrt{d}}\right)$ such that $p = O(1)$, there exists some numerical constant $c > 0$, 
such that whenever $\kappa > c\sqrt[4]{\frac{\log N}{d}}$, 
the clustering error rate of TIP-SC is less than $O\left(\frac{(1+1/\rho)\log N}{\kappa^2d}\right)$ with probability at least $1 - \me^{-\Omega(\log N)}$.
\end{theorem}

Parameter selection is often critical for the success of algorithms.
The above result suggests that a dense graph ($p = O(1)$) is usually a good choice,
which is quite different with SDP.  

In this regime, the above result indicates that the algorithm works correctly in fairly broad conditions compared with existing analysis for SC. 
A fascinating insight revealed by the above theorem is that 
even when the number of samples $N \ll d$, we can succeed to cluster the data set,
which demonstrates the commonly accepted opinion that $\rho > 1$ is necessary for SC is partially inaccurate.

To clarify the condition on $\kappa$, namely on affinity, 
assume these two subspaces overlap in a smaller subspace of dimension $s$,
but are orthogonal to each other in the remaining directions.
In this case, the affinity between the two subspaces is equal to $\sqrt{s/d}$.
Our assumption on $\kappa$ indicates that subspaces can have intersections of almost all dimensions, i.e., $s = (1 - o(1))d$. 
In contrast, previous works~\citep{soltanolkotabi2012geometric, soltanolkotabi2014robust} imposes that the overlapping dimension should obey $s = o(1)d$, 
so that the subspaces are practically orthogonal to each other.

In the noisy case, we assume each data point is of the form
\begin{equation} \label{eq:noise}
\vy = \vx + \vz,
\end{equation} 
where $\vx$ denotes the clean data used in the above theorem,
and $\vz \sim \mathcal{N}(0, \frac{\sigma^2}{n}\vI)$ is an independent stochastic noise term.
We have the following robustness guarantee for TIP-SC. 

\begin{theorem} \label{thm:noise}
Choosing $\tau = O\left(\frac{1}{\sqrt{d}}\right)$ such that $p = O(1)$, there exists some numerical constant $c, \sigma^* > 0$, 
such that whenever $\kappa > c\sqrt[4]{\frac{\log N}{d}}$ and $\sigma < \sigma^*$, 
the clustering error rate of TIP-SC is less than $O\left(\frac{(1+\sigma^2d/n)^2(1+1/\rho)\log N}{\kappa^2d}\right)$ with probability at least $1 - \me^{-\Omega(\log N)}$.
\end{theorem}
The proof is similar to that of Theorem~\ref{thm:result}, and both are deferred to Section~\ref{sec:proof}.

\section{Numerical Experiments}
\label{sec:exp}
In this section, we perform numerical experiments validating our main results.
We evaluate the algorithm and theoretical results based on the clustering accuracy.
The impacts of $\kappa, \rho, p, q$ on the clustering accuracy are demonstrated.
Besides, we also show the efficiency of TIP-SC in the presence of noise.

According to the definition of semi-random model, 
to save computation and for simplicity, the data are generated by the following steps.
\begin{itemize}
\item[1)] Given $d \ll n$ and ${\rm aff} = \sqrt{s/d}$, 
define $\ve_i \in \mathbb{R}^n$, whose entries are zero but the $i$-th entry is one. 
Let $\vU_1 = [\ve_1, \ve_2, \ldots, \ve_d]$ be the orthonormal basis for subspace for $S_1$,
and $\vU_2 = [\ve_{d-s+1}, \ve_{d-s+2}, \ldots, \ve_{2d-s}]$ be the orthonormal basis for subspace for $S_2$,
such that the affinity between $S_1$ and $S_2$ is $\sqrt{s/d}$.
\item[2)] Given $N = \rho d$, generate $N$ vectors $\va_1, \va_2, \ldots, \va_N \in \mathbb{R}^d$ independently from $\mathcal{N}(0, \frac{1}{d}\vI)$.
Let $\vx_i = \vU_1 \frac{\va_i}{\|\va_i\|}$ for $1 \le i \le N/2$ 
and $\vx_i = \vU_2 \frac{\va_i}{\|\va_i\|}$ for $N/2+1 \le i \le N$.
\item[3)] In the presence of noise, given $\sigma > 0$, generate $N$ random noise terms $\vz_1, \vz_2, \ldots, \vz_N \in \mathbb{R}^n$ independently from $\mathcal{N}(0, \frac{\sigma^2}{n}\vI)$.
Let the normalized data of $\vx_i + \vz_i$ be the input of Algorithm~\ref{alg:TIP-SC}.
\end{itemize}

Since there are too many factors we need to consider, 
we always observe the relation between two concerned quantities,
while keep others being some predefined typical values,
i.e., $d^* = 100, n^* = 5000, \kappa^* = 1 - \sqrt{1/2}\ (s^* = d/2), \rho^* = 1$,
and $\tau$ is chosen to be $\tau^*$ such that the connection rate $\frac{p+q}{2} = 0.2$. 
We conduct the experiments in noiseless situations, 
except the last one which tests the robustness of Algorithm~\ref{alg:TIP-SC}.
Moreover, the curves are plotted by $100$ trials in each experiment,
while the mean and the standard deviation are represented by line and error bar, respectively.
We can find that the randomness is eliminated in all experiments when the error rate is small.

\begin{figure}[t]
\begin{center}
\includegraphics[width=0.6\textwidth]{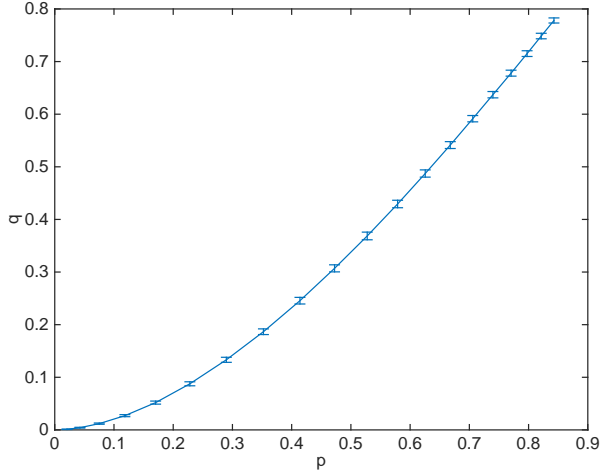}
\caption{The relation between $p$ and $q$, when $d = 100, n = 5000, \kappa = 1 - \sqrt{1/2}\ (s = d/2), \rho = 1$.}\label{fig:p-q}
\end{center}
\end{figure}

\begin{figure}[t]
\begin{center}
\includegraphics[width=0.6\textwidth]{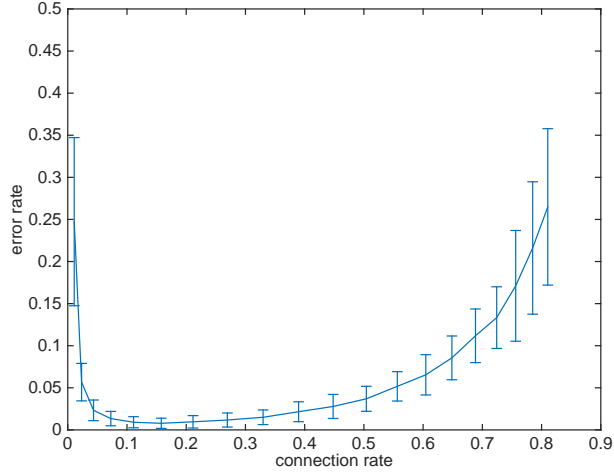}
\caption{This figure demonstrates the clustering error rate versus the connection rate ($\frac{p+q}{2}$) in a general interval, when $d = 100, n = 5000, \kappa = 1 - \sqrt{1/2} (s = d/2), \rho = 1$.}\label{fig:e-p}
\end{center}
\end{figure}

It is obvious that $p$ will decrease simultaneously if $q$ decreases by increasing $\tau$,
which is also demonstrated in Figure~\ref{fig:p-q}.
Combining the result of the second experiment (c.f. Figure~\ref{fig:e-p}),
we can find that it is better to make $p, q$ both large than to choose $q = 0$,
although $q = 0$ is suggested by SDP,
which is consistent with our result,
while shows that SDP is somewhat inadequate for SC.

\begin{figure}[t]
\begin{center}
\includegraphics[width=0.6\textwidth]{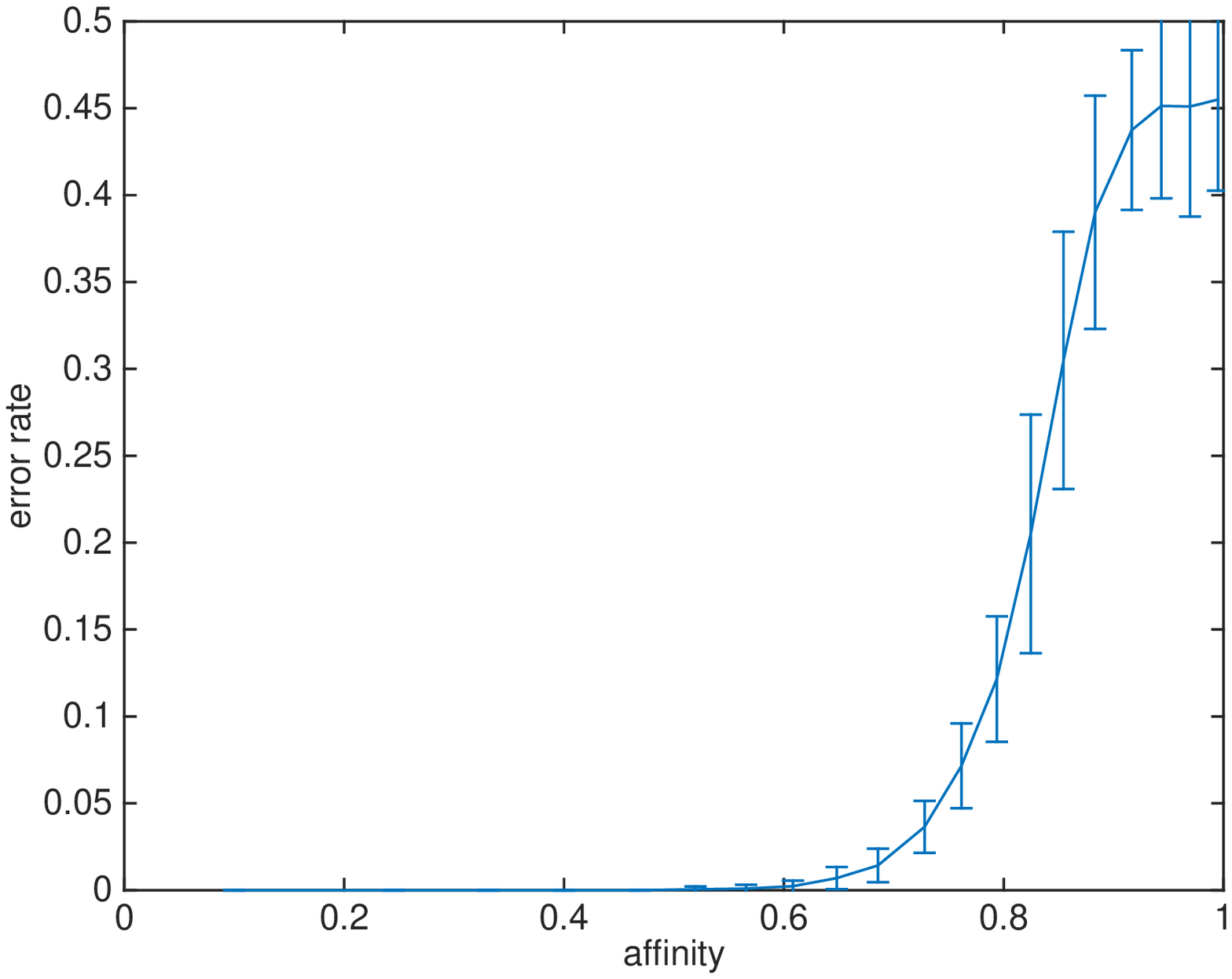}
\caption{This figure demonstrates the clustering error rate versus the affinity in a general interval, when $d = 100, n = 5000, \rho = 1, \frac{p+q}{2} = 0.2$.}\label{fig:e-k}
\end{center}
\end{figure}

\begin{figure}[t]
\begin{center}
\includegraphics[width=0.6\textwidth]{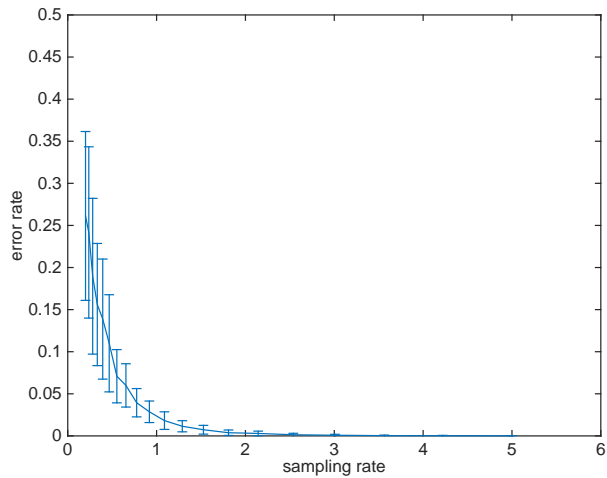}
\caption{This figure demonstrates the clustering error rate versus the sampling rate $\rho$ in a general interval, when $d = 100, n = 5000, \kappa = 1 - \sqrt{1/2}\ (s = d/2), \frac{p+q}{2} = 0.2$.}\label{fig:e-r}
\end{center}
\end{figure}

\begin{figure}[t]
\begin{center}
\includegraphics[width=0.6\textwidth]{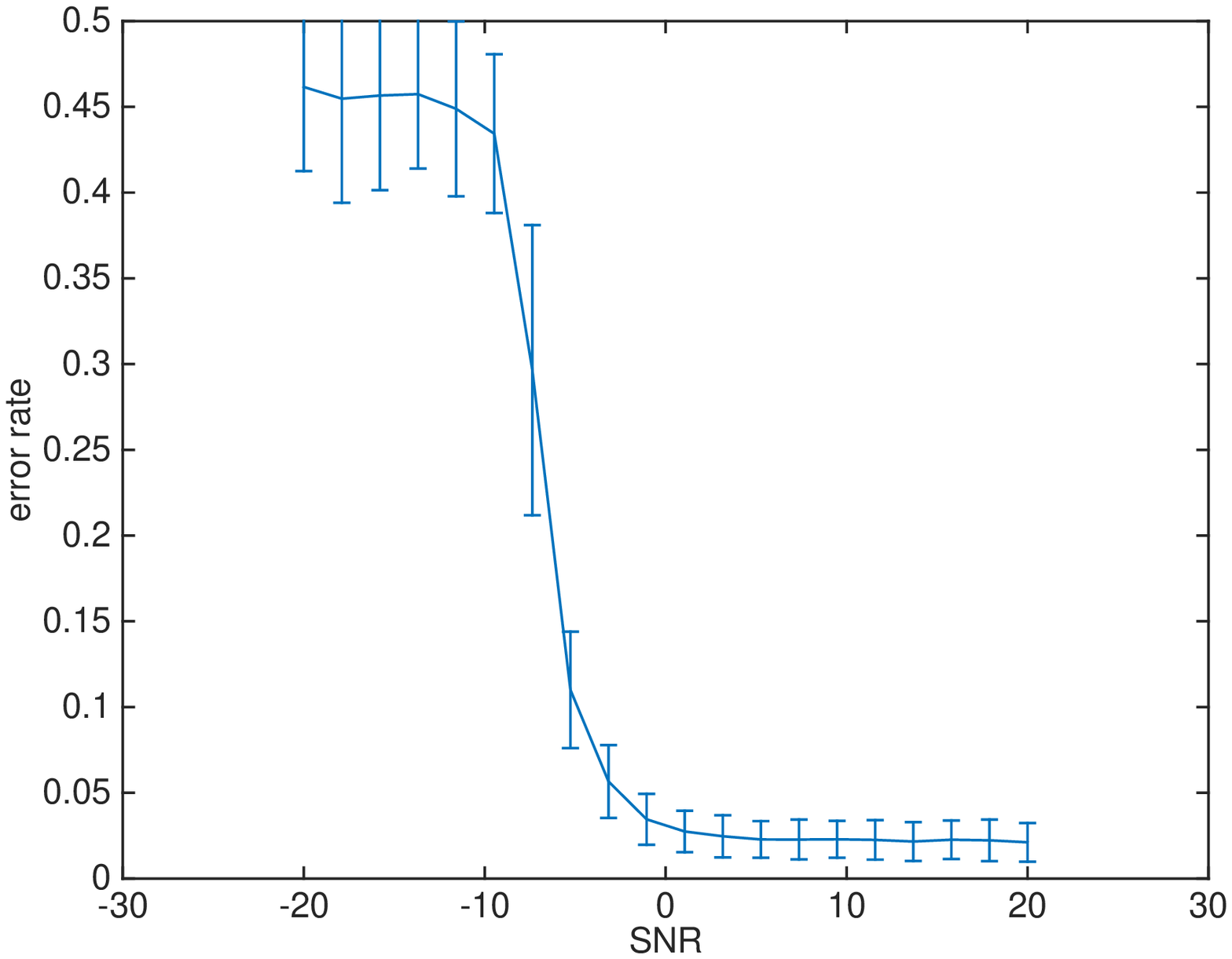}
\caption{This figure demonstrates the clustering error rate versus the SNR in a general interval, when $d = 100, n = 5000, \kappa = 1 - \sqrt{1/2}\ (s = d/2), \rho = 1, \frac{p+q}{2} = 0.2$.}\label{fig:noise}
\end{center}
\end{figure}

In the third and fourth experiments, we inspect the impacts of affinity and sampling rate on the performance of TIP-SC. 
From Figure~\ref{fig:e-k} and Figure~\ref{fig:e-r}, 
the claim that SC works well in fairly broad conditions is verified.
In addition, according to \eqref{eq:noise}, we have 
\begin{equation*}
{\rm SNR} = 10\log \frac{1}{\sigma^2},
\end{equation*}
then the last experiment (c.f. Figure~\ref{fig:noise}) shows that the algorithm is robust even though SNR is low. 

\section{Proof of Main Results}
\label{sec:proof}

\subsection{Proof of Theorem~\ref{thm:result}}

Recall the definition of $\vu, \vv, \vw, \vW$ in Section~\ref{sec:notation},
and notice that analyzing the error rate, denoted by $\gamma$, is equivalent to studying the difference between $\vw$ and $\vv$.
Without loss of generality we may assume that $\langle\vw, \vv\rangle > 0$,
thus the error rate is exactly
$$
\gamma = \frac{1}{4}\left\|\frac{1}{\sqrt{N}}{\rm sgn}(\vw) - \vv\right\|_2^2.
$$
To estimate $\gamma$, it suffices to bound the distance between $\vu, \vv$ and $\vW$.

By simple geometric consideration, we have

\begin{align*}
\left\|\frac{1}{\sqrt{N}}{\rm sgn}(\vw) - \vv\right\|_2 \le& 2\|\vP_\vw\vv - \vv\|_2 \\
\le& 2(\|\vP_\vW\vv - \vv\|_2 + \|\vP_\vw\vv - \vP_\vW\vv\|_2) \\
=& 2(\|\vP_\vW\vv - \vv\|_2 + |\langle\overline{\vP_\vW\vu}, \vv\rangle|) \\
\le& 2(\|\vP_\vW\vv - \vv\|_2 + \|\overline{\vP_\vW\vu} - \vu\|_2) \\
\le& 2\|\vv - \vP_\vW\vv\|_2 + 4\|\vu - \vP_\vW\vu\|_2,
\end{align*}
where $\overline{\vP_\vW\vu}$ denote the normalization of $\vP_\vW\vu$.
Moreover, for any $\lambda, \vx$, we have 
\begin{equation*}
\|\vA\vx - \lambda\vx\|_2 \ge (\lambda - \lambda_3(\vA))\|\vx - \vP_\vW\vx\|_2,
\end{equation*}
where $\lambda_3(\vA)$ denotes the third largest eigenvalue of $\vA$.

Summing up, for $\lambda_1, \lambda_2 > \lambda_3(\vA)$,
$$
\gamma = \frac{1}{4}\left\|\frac{1}{\sqrt{N}}{\rm sgn}(\vw) - \vv\right\|_2^2 \lesssim \frac{\|\vA\vu - \lambda_1\vu\|_2^2}{(\lambda_1 - \lambda_3(\vA))^2} + \frac{\|\vA\vv - \lambda_2\vv\|_2^2}{(\lambda_2 - \lambda_3(\vA))^2},
$$
Considering that $\ex \langle \vA\vu, \vu \rangle = p(N/2-1)+qN/2$, 
we expect $\lambda_1 = p(N/2-1)+qN/2$ is a good choice.
Similarly, choose $\lambda_2 = p(N/2-1)-qN/2$.

From above discussion, to estimate $\gamma$ we need to: 
\begin{itemize}
\item Prove $\|\vA\vu - \lambda_1\vu\|_2$ and $\|\vA\vv - \lambda_2\vv\|_2$ are sufficiently small (see Lemma~\ref{lem:pq} and Lemma~\ref{lem:count}).
\item Prove $\lambda_1 - \lambda_3(\vA)$ and $\lambda_2 - \lambda_3(\vA)$ are sufficiently large,
which is equivalent to showing $p-q$ is large enough (see Lemma~\ref{lem:pq}) and $\lambda_3(\vA)$ is small enough (see Lemma~\ref{lem:eig}).
\end{itemize}

Before proceeding, we analyze the adjacent matrix $\vA$ based on the conditional independence property,
and provide probability estimations used in the proof of Theorem~\ref{thm:result}.
Specifically, this refers to if conditioned on $\vx_i, i \in \mathcal{S}$ for some subset $\mathcal{S}$ of $[N]$,
$A_{ij}$,  for $j \in \mathcal{S}^c$, are functions of $\vx_j$, respectively, 
and then are independent from each other. 

Moreover, recalling the definition of $\vx_i, \va_i$, 
on the collection of events $\mathcal{E}(t)$ given by the intersection of
\begin{align*}
\{\forall i, |\|\va_i\| - 1| < t\} \\ 
\left\{\forall i, \left|\sum_k \lambda_k^2 a_{ik}^2 - \frac{\sum_k \lambda_k^2}{d}\right| < t\right\} \\ 
\{\forall i \ne j, |\langle \vx_i, \vx_j \rangle| < t\},
\end{align*}
if conditioned on $\vx_i, i \in \mathcal{S}$, $A_{ij}$,  for $j \in \mathcal{S}^c$ are \emph{nearly identically distributed},
and for some $j \in \mathcal{S}^c$, $A_{ij}$,  for $i \in \mathcal{S}$ are \emph{nearly independent} from each other,
which will be explained and employed many times in the following analysis. 
According to Lemma~\ref{lem:ineq_norm} and Lemma~\ref{lem:angle}, 
there exist some constants $c_1, c_2 > 0$, such that
$$
\pr\left(\mathcal{E}\bigg(c_1\sqrt{\frac{\log N}{d}}\bigg)\right) > 1 - \me^{-c_2\log N}.
$$
For simplicity, use $\mathcal{E}$ to denote $\mathcal{E}\left(c_1\sqrt{\frac{\log N}{d}}\right)$.
In this work, we will always analyze the spectral method on the canonical event set $\mathcal{E}$.

Let 
$$
\Phi(t) := \int_{|x| > t} \frac{1}{\sqrt{2\pi}}\me^{-\frac{x^2}{2}} {\rm d}x,
$$
then
\begin{lemma} \label{lem:p}
All $p_i$ are equal, and there exist constants $c_1, c_2 > 0$, such that
$$
p = p_i > \Phi(\tau_d(1+t)) - \me^{-c_2 \log N},
$$
where $\tau_d := \sqrt{d}\tau$ and $t = c_1\sqrt{\frac{\log N}{d}}$.
\end{lemma}

\begin{proof}
Conditioned on $\vx_i$, for $\vx_j \in S$
\begin{equation*}
p_i := \pr\left(|\langle \vx_i, \vx_j\rangle| \ge \tau \Bl \vx_i\right) = \pr\left(|\langle \overline{\va}_i, \overline{\va}_j\rangle| \ge \tau \Bl \va_i\right),
\end{equation*}
where $\va_i, \va_j\overset{\mathrm{ind.}}{\sim} \mathcal{N}(\bm{0}, \frac{1}{d}\vI_d) \in \mathbb{R}^d$, 
and $\overline{\va}_i, \overline{\va}_j$ denote the normalization with $\vx_i = \vU\overline{\va}_i, \vx_j = \vU\overline{\va}_j$.
According to the independence between $\va_i, \va_j$ and the rotational invariance property of Gaussian random vectors,
it is obviously that all $p_i$ are equal.
Moreover, we have
\begin{align*}
&\pr\left(|\langle \vx_i, \vx_j\rangle| \ge \tau \Bl \vx_i\right) \\
=& 2\pr\left(\langle \overline{\va}_i, \va_j\rangle \ge \tau\|\va_j\| \Bl \va_i\right) \\
\ge& 2\pr\left(\langle \overline{\va}_i, \va_j\rangle \ge \tau(1+t) \Bl \va_i\right) - 2\pr\left(\|\va_j\| > 1+t \right) \\
>& \Phi\left(\sqrt{d}\tau(1+t)\right) - \me^{-c_2 \log N}, 
\end{align*}
since $\langle \overline{\va}_i, \va_j\rangle \sim \mathcal{N}(0, \frac{1}{d})$ is a Gaussian random variable independent with $\va_i$,
and 
$$\pr\left(\|\va_j\| > 1+t \right) < \me^{-c_2 \log N}$$ according to Lemma~\ref{lem:ineq_norm}.
\end{proof}

\begin{lemma} \label{lem:q}
There exist constants $c_1, c_2 > 0$, such that for $t = c_1\sqrt{\frac{\log N}{d}}$, on $\mathcal{E}(t)$, we have 
$$
\Phi\left(\frac{\tau_d(1+t)}{{\rm aff}^2 - t}\right) - \me^{-c_2 \log N} < q_i < \Phi\left(\frac{\tau_d(1-t)}{{\rm aff}^2 + t}\right) + \me^{-c_2 \log N},
$$
where $\tau_d := \sqrt{d}\tau$.
\end{lemma}

\begin{proof}
According to Remark 5 in~\citet{li2017restricted}, we can choose $\vU_1, \vU_2$ such that
\begin{align*}
\vU_2^{\top}\vU_1 
= \left[ \begin{array}{ccc}
     \lambda_1 & & \\ & \ddots&  \\ & & \lambda_{d} \end{array} \right].
\end{align*}
Without loss of generality, assume that $\vx_i \in S_1, \vx_j \in S_2$, then
\begin{equation*}
\pr\left(|\langle \vx_i, \vx_j\rangle| \ge \tau \Bl \vx_i\right) = \pr\left(|\langle \vU_1\overline{\va}_i, \vU_2\overline{\va}_j\rangle| \ge \tau \Bl \va_i\right),
\end{equation*}
where $\va_i, \va_j\overset{\mathrm{ind.}}{\sim} \mathcal{N}(\bm{0}, \frac{1}{d}\vI_d) \in \mathbb{R}^d$, 
and $\overline{\va}_i, \overline{\va}_j$ denote the normalization with $\vx_i = \vU_1\overline{\va}_i, \vx_j = \vU_2\overline{\va}_j$.
In addition, the definition of $\mathcal{E}(t)$ gives,
\begin{equation*}
\left|\|\vU_2^{\top}\vU_1\overline{\va}_i\|^2 - {\rm aff}^2\right| < t,
\end{equation*}
then according to Lemma~\ref{lem:ineq_norm}
\begin{align*}
q_i =& \pr\left(|\langle \vx_i, \vx_j\rangle| \ge \tau \Bl \vx_i\right) \\
=& \pr\left(|\langle \vU_2^{\top}\vU_1\overline{\va}_i, \overline{\va}_j\rangle| \ge \tau \Bl \va_i\right) \\
\ge& 2\pr\left(\langle  \vU_2^{\top}\vU_1\overline{\va}_i, \va_j\rangle \ge \tau(1+t) \Bl \va\right) - 2\pr\left(\|\va_j\| > 1+t \right) \\
>& \Phi\left(\frac{\sqrt{d}\tau(1+t)}{{\rm aff}^2 - t}\right) - \me^{-c_2 \log N},
\end{align*}
and similarly,
$$
q_i < \Phi\left(\frac{\sqrt{d}\tau(1-t)}{{\rm aff}^2 + t}\right) + \me^{-c_2 \log N}.
$$
\end{proof}

Specifically, according to the above two lemmas about $p_i, q_i$, we can easily get the following lemma.

\begin{lemma} \label{lem:pq}
Choose $\tau_d = O(1)$, then $p = \Omega(1)$.
Moreover, on $\mathcal{E}$, there exists some constant $c > 0$, 
such that if $\kappa = 1 - {\rm aff}^2 > c\sqrt{\frac{\log N}{d}}$,
\begin{equation*}
p - q \gtrsim \kappa,
\end{equation*}
and
\begin{equation*}
\frac{1}{N} \sum_{i} (q_i - q)^2 \lesssim \frac{\log N}{d}.
\end{equation*}
\end{lemma}

Having finished the calculation about the probability of each entry, 
we now turn to the overall properties of $\vA$.

\begin{lemma} \label{lem:count}
Conditioned on $\vx_i$, for any $t > 0$
$$
\pr\left(\left|\frac{1}{N/2-1}\sum_{j : \vx_j \in S}A_{ij} - p\right| > t\right) < \me^{-\frac{t^2(N/2-1)}{p + \frac{1}{3}t}}, 
$$
and
$$
\pr\left(\left|\frac{1}{N/2}\sum_{j : \vx_j \in S'}A_{ij} - q_i\right| > t\right) < \me^{-\frac{t^2N/2}{q_i + \frac{1}{3}t}}. 
$$
\end{lemma}

\begin{proof}
Given $\vx_i$, it can be easily checked that the angels between $\vx_j$ and $\vx_i$ are independent with each other,
then $A_{ij}$ are conditionally independent Bernoulli random variables.
Hence, according to Lemma~\ref{lem:Bernstein}, the results is obvious.
\end{proof}

In the next lemma, we will analyze the eigenvalue of $\vA$.

\begin{lemma} \label{lem:eig}
For $t = c_1\sqrt{\frac{\log N}{d}}$, on $\mathcal{E}(t)$, with probability at least $1 - \me^{-c_2\log N}$, 
$$
\lambda_3(\vA) < c\sqrt{Np\log N + N^2p^2t}, 
$$
where $\lambda_3(\vA)$ denotes the third largest eigenvalue of $\vA$,
and $c, c_1, c_2 > 0$ are some constants.
\end{lemma}

\begin{proof}
We transfer the estimation of $\lambda_3(\vA)$ to bounding $\lambda_{\max}\left(\vE\right)$ using Lemma~\ref{lem:eig_property}, i.e., 
\begin{align*}
\lambda_3(\vA) =& \min_{S_{N-2}} \max_{\vx \in S_{N-2}} \vx^{\top}\vA\vx \\
\le& \max_{\vx \in {\rm span}(\vu, \vv)^{\perp}} \vx^{\top}\vA\vx \\
=& \max_{\vx \in {\rm span}(\vu, \vv)^{\perp}} \vx^{\top}\vE\vx \\
\le& \max_{\vx : \|\vx\| = 1} \vx^{\top}\vE\vx \\
=& \lambda_{\max}\left(\vE\right),
\end{align*}
where $\vu, \vv$ are defined in Section~\ref{sec:notation}, 
and 
$$\vE = \vA - (p+q)N/2\vu\vu^{\top} - (p-q)N/2\vv\vv^{\top},$$
then $E_{ij} = -p$, if $i = j$, $E_{ij} = A_{ij} - p$, if $\vx_j \in S$ and $E_{ij} = A_{ij} - q$, if $\vx_j \in S'$.

The analysis of $\lambda_{\max}\left(\vE\right)$ is based on the decoupling technique.
According to Lemma~\ref{lem:dec}, 
let $\mathcal{S}$ be a random subset of $[N]$ with average size $N/2$, then
\begin{align*}
\lambda_{\max}\left(\vE\right) &= \sup_{\vx : \|\vx\| = 1} \vx^{\top}\vE\vx \\
&= \sup_{\vx : \|\vx\| = 1} \left(\sum_i x_i^2E_{ii} + \sum_{i \ne j} x_ix_jE_{ij}\right) \\
&= -p + \sup_{\vx : \|\vx\| = 1} 4\ex_{\mathcal{S}} \sum_{i \in \mathcal{S}, j \in \mathcal{S}^c} x_ix_jE_{ij} \\
&\le -p + 4\ex_{\mathcal{S}} \sup_{\vx : \|\vx\| = 1} \sum_{i \in \mathcal{S}, j \in \mathcal{S}^c} x_ix_jE_{ij} \\
&\le -p + 4\ex_{\mathcal{S}} \|\vE_{\mathcal{S}, \mathcal{S}^c}\|_{\mathrm{op}},
\end{align*}
where $\vE_{\mathcal{S}, \mathcal{S}^c}$ denotes the sub-matrix of $\vE$ including the rows from $\mathcal{S}$ and columns from $\mathcal{S}^c$,
and $\|\cdot\|_{\mathrm{op}}$ denotes the operator norm.

To analyze $\|\vE_{\mathcal{S}, \mathcal{S}^c}\|_{\mathrm{op}}$, we first condition on  $\mathcal{S}$ and $\vx_j, j \in \mathcal{S}^c$, 
and for $i \in \mathcal{S}$, let $\Gamma_i := \vE_{i, \mathcal{S}^c}, \vR_i := \ex \Gamma_i^{\top}\Gamma_i$, and $L := \max_i\|\Gamma_i\|^2$, 
then $\Gamma_i$ are independent with each other.
On $\mathcal{E}$, 
\begin{align*}
L =& \max_i\|\Gamma_i\|^2 \\
=& \max_i \sum_{j \in \mathcal{S}^c} E_{ij}^2 \\
\lesssim& Np.
\end{align*}
Moreover, for the diagonal entries of $\vR_i$, 
\begin{align*}
\vx_j \in S : \ex \Gamma_{ij}\Gamma_{ij} =& \ex (A_{ij} - p)^2 = p(1-p) \le p, \\
\vx_j \in S' : \ex \Gamma_{ij}\Gamma_{ij} =& \ex (A_{ij} - q)^2 = q(1-q)+(q_j-q)^2 \le p.
\end{align*}
On the other hand, for the off-diagonal entries of $\vR_i$, 
if $\vx_j, \vx_k \in S$,
\begin{align*}
\left|\ex \Gamma_{ij}\Gamma_{ik}\right| = \left|\ex (A_{ij} - p)(A_{ik} - p)\right| \lesssim p^2t,
\end{align*}
since $\langle \vx_j, \vx_k \rangle \le t$.
With similar analysis on the cases $\vx_j \in S', \vx_k \in S$ and $\vx_j, \vx_k \in S'$,
we have the off-diagonal entries of $\vR_i$ are less than $p^2t$.
Hence,
\begin{equation*}
\lambda_{\max}\left(\vR_i\right) \lesssim p + Np^2t =: \lambda.
\end{equation*}
and Lemma~\ref{lem:Matrix_Bernstein} gives, for $0 < \theta < 3/L$,
\begin{align*}
\log \ex \exp\bigg(\theta \sum_i\left(\Gamma_i^{\top}\Gamma_i - \vR_i\right)\bigg) &= \sum_i\log \ex \exp\bigg(\theta\left(\Gamma_i^{\top}\Gamma_i - \vR_i\right)\bigg) \\
&\preccurlyeq \sum_i\frac{\theta^2/2}{1-\theta L/3} \ex \left(\Gamma_i^{\top}\Gamma_i - \vR_i\right)^2 \\
&\preccurlyeq \sum_i\frac{\theta^2/2}{1-\theta L/3} L\vR_i.
\end{align*}
Then
\begin{align*}
&\pr\left(\ex_{\mathcal{S}} \bigg(\|\vE_{\mathcal{S}, \mathcal{S}^c}\|_{\mathrm{op}}^2 - \sum_i\lambda_{\max}\left(\vR_i\right)\bigg) > t\right) \\
<& \inf_{\theta} \me^{-\theta t}\ex \exp\left(\ex_{\mathcal{S}} \theta \bigg(\|\vE_{\mathcal{S}, \mathcal{S}^c}\|_{\mathrm{op}}^2 - \sum_i\lambda_{\max}\left(\vR_i\right)\bigg)\right) \\
\le& \inf_{\theta} \me^{-\theta t}\ex\ex_{\mathcal{S}} \exp\left(\theta \bigg(\|\vE_{\mathcal{S}, \mathcal{S}^c}\|_{\mathrm{op}}^2 - \sum_i\lambda_{\max}\left(\vR_i\right)\bigg)\right) \\
\le& \inf_{\theta} \me^{-\theta t}\ex \mathrm{Tr}\left(\exp\bigg(\theta \sum_i\left(\Gamma_i^{\top}\Gamma_i - \vR_i\right)\bigg)\right) \\
\le& \inf_{\theta} \me^{-\theta t}\ex \mathrm{Tr}\left(\exp\bigg(\sum_i \frac{\theta^2/2}{1-\theta L/3}L \vR_i\bigg)\right) \\
\le& \inf_{\theta} \me^{-\theta t} N\exp\left(\sum_i \frac{\theta^2/2}{1-\theta L/3}L \lambda_{\max}\left(\vR_i\right)\right) \\
\lesssim& N\exp\left(\frac{-t^2/2}{NL \lambda + Lt/3}\right).
\end{align*}
Hence, with probability at least $1 - \me^{-c_2\log N}$,
\begin{align*}
\ex_{\mathcal{S}} \|\vE_{\mathcal{S}, \mathcal{S}^c}\|_{\mathrm{op}} &\le \sqrt{\ex_{\mathcal{S}} \|\vE_{\mathcal{S}, \mathcal{S}^c}\|_{\mathrm{op}}^2} \\
&\le c\sqrt{N\lambda + L\log N} \\
&= c\sqrt{Np\log N + N^2p^2t}.
\end{align*}
Summing up,
\begin{equation*}
\lambda_3(\vA) \le \lambda_{\max}\left(\vE\right) < c\sqrt{Np\log N + N^2p^2t}.
\end{equation*}
We conclude the proof.
\end{proof}

Now, we have all the ingredients for the proof of Theorem~\ref{thm:result}.

\begin{proof}[Proof of Theorem~\ref{thm:result}]
We begin with some inequalities for estimating the error.
We have
\begin{align*}
&\|\vA\vu - (p(N/2-1)+qN/2)\vu\|_2^2 \\
=& \frac{1}{N}\sum_i \bigg(\sum_{j : \vx_j \in S}A_{ij} - p(N/2-1) + \sum_{j : \vx_j \in S'}A_{ij} - qN/2\bigg)^2 \\
\le& \frac{3}{N}\sum_i \bigg(\sum_{j : \vx_j \in S}A_{ij} - p(N/2-1)\bigg)^2 \\ 
&+ \frac{3}{N}\sum_i \bigg(\sum_{j : \vx_j \in S'}A_{ij} - q_iN/2\bigg)^2 + \frac{3}{N}\sum_i (q_iN/2 - qN/2)^2.
\end{align*}
According to Lemma~\ref{lem:count}, for all $1 \le i \le N$, we have, with probability at least $1 - \exp(-\Omega(\log N))$,  
\begin{equation*}
\bigg(\sum_{j : \vx_j \in S}A_{ij} - p(N/2-1)\bigg)^2 \lesssim N\log N,
\end{equation*}
and
\begin{equation*}
\bigg(\sum_{j : \vx_j \in S'}A_{ij} - q_iN/2\bigg)^2 \lesssim N\log N.
\end{equation*}
On the other hand, Lemma~\ref{lem:pq} gives, with probability at least $1 - \exp(-\Omega(\log N))$,
\begin{equation*}
\frac{3}{N}\sum_i (q_iN/2 - qN/2)^2 \lesssim \rho N\log N.
\end{equation*}
Summing up, we have, with probability at least $1 - \exp(-\Omega(\log N))$,
$$
\|\vA\vu - (p(N/2-1)+qN/2)\vu\|_2^2 \lesssim (1+\rho)N\log N.
$$
Similarly, 
with probability at least $1 - \exp(-\Omega(\log N))$,
$$
\|\vA\vv - (p(N/2-1)-qN/2)\vv\|_2^2 \lesssim (1+\rho)N\log N.
$$

According to Lemma~\ref{lem:eig}, for $t = O\left(\sqrt{\frac{\log N}{d}}\right)$, with probability at least $1 - \exp(-\Omega(\log N))$, the third largest eigenvalue of $\vA$ satisfies
$$
\lambda_3(\vA) \lesssim \sqrt{Np\log N + N^2p^2t} = O\left(N\sqrt[4]{\frac{\log N}{d}}\right).
$$

With these estimations at hand,
recall
\begin{align*}
\gamma \lesssim& \frac{\|\vA\vu - (p(N/2-1)+qN/2)\vu\|_2^2}{|p(N/2-1)+qN/2 - \lambda_3(\vA)|^2} + \frac{\|\vA\vv - (p(N/2-1)-qN/2)\vv\|_2^2}{|p(N/2-1)-qN/2 - \lambda_3(\vA)|^2}.
\end{align*}
Lemma~\ref{lem:pq} gives $p \pm q \gtrsim 1 - {\rm aff}^2$,
then we have 
\begin{equation*}
\gamma \lesssim \left(\frac{\sqrt{(1+\rho)N\log N}}{N\Big(1 - {\rm aff}^2 - \sqrt[4]{\frac{\log N}{d}}\Big)}\right)^2 \sim \frac{(1+\rho)\log N}{\kappa^2N}.
\end{equation*}
We conclude the proof.
\end{proof}

\subsection{Proof of Theorem~\ref{thm:noise}}

Robustness analysis can be completed by following the similar analysis method.
We provide the differences in the analysis of noise, while omit the details.

Here, we only need to pay attention to the changes of Lemma~\ref{lem:pq}, Lemma~\ref{lem:count}, and Lemma~\ref{lem:eig}, when adding noise.
Notice that the noise terms do not destroy the wonderful conditional independence property,
then it's obvious that except the estimation for $p-q$, all other bounds still hold in a similar way.
Through simple calculation, the contribution of noise has the form
$$
p - q \gtrsim \frac{\kappa}{1+\sigma^2d/n}.
$$
Taking this change into account, we can get the result easily.

\section{Conclusion}

This paper establish a theory to analyze spectral method for Random Geometry Graph constructed by data points from Union of Subspaces.
Based on this theory, we demonstrate the efficiency of Subspace Clustering in fairly broad conditions.
To the best of our knowledge, the clustering accuracy has not been shown in the prior literature.
The insights and analysis techniques developed in this paper might also have implications for other Random Geometry Graph.

Moving forward, one issue is to understand UoS-RGG constructed by more complex strategy, such as SSC.
Additionally, ideally one would desire an exact recovery by spectral method,
which needs entrywise analysis.
We leave these to future investigation.

\appendix
\section{Auxiliary Lemmas}

In this subsection, we introduce some well-known results about Gaussian, Bernoulli random variables, and matrices \citep{vershynin2010introduction}, 
which shall be used to analyze the properties of the adjacent matrix $\vA$.
We omit the proof for most of them.

\begin{lemma}[Concentration in Gauss space \citep{ledoux2001concentration}] \label{lem:ineq_Lip} 
Let $f$ be a real valued Lipschitz function on $\mathbb{R}^n$ with Lipschitz constant $K$, i.e.,
$$
\left| f(\vx_1)-f(\vx_2) \right| \le K\left\| \vx_1-\vx_2 \right\|
$$ 
for any $\vx_1, \vx_2 \in \mathbb{R}^n$ (such functions are also called K-Lipschitz). 
Let $X \sim \mathcal{N}(\bm{0}, \vI_n)$ be the standard Gaussian random vector in $\mathbb{R}^n$, 
then for every $t > 0$, one has
$$
\mathbb{P} \left( f(\vx)-\mathbb{E}f(\vx) > t \right) < e^{-\frac{t^2}{2K^2}}.
$$
\end{lemma}

\begin{lemma} \label{lem:ineq_norm}
Assume $\va \sim \mathcal{N}(\bm{0}, \frac{1}{d}\vI_d) \in \mathbb{R}^d$, then for any $t > 0$
\begin{equation} \label{eq:norm}
\pr\left(|\|\va\| - 1| > t + \frac{2}{\sqrt{d}} \right) < 2\me^{-\frac{dt^2}{2}}.
\end{equation}
Moreover, for $0 \le \lambda_1, \ldots, \lambda_d \le 1$ and $t > 0$
\begin{equation} \label{eq:square}
\pr\left(\left|\sum_i \lambda_i^2 a_i^2 - \frac{\sum_i \lambda_i^2}{d}\right| > 2t\sqrt{\frac{\sum_i \lambda_i^2}{d}} + t^2 + \frac{4}{d}\right) < 2\me^{-\frac{dt^2}{2}}.
\end{equation}
\end{lemma}

\begin{proof}
Let
\begin{equation*}
f(\vx) = \sqrt{\sum_i \lambda_i^2 x_i^2},
\end{equation*}
then by calculation
\begin{equation*}
\|\nabla f(\vx)\| = \sqrt{\frac{\sum_i \lambda_i^4 x_i^2}{\sum_i \lambda_i^2 x_i^2}} \le 1.
\end{equation*}
Hence, $f(\vx)$ is $1-Lipschitz$ and according to Lemma~\ref{lem:ineq_Lip}, we have
\begin{equation*}
\pr\left(\sqrt{\sum_i \lambda_i^2 a_i^2} - \ex \sqrt{\sum_i \lambda_i^2 a_i^2} > t \right) < \me^{-\frac{dt^2}{2}}.
\end{equation*}
Take $f(x) = -\sqrt{\sum_i \lambda_i^2 x_i^2}$, then similarly
\begin{equation*}
\pr\left(\sqrt{\sum_i \lambda_i^2 a_i^2} - \ex \sqrt{\sum_i \lambda_i^2 a_i^2} < -t \right) < \me^{-\frac{dt^2}{2}}.
\end{equation*}
Moreover, $\left(\ex \sqrt{\sum_i \lambda_i^2 a_i^2}\right)^2 \le \ex \sum_i \lambda_i^2 a_i^2 = \frac{\sum_i \lambda_i^2}{d}$ and
\begin{align*}
&\bigg(\ex \sqrt{\sum_i \lambda_i^2 a_i^2}\bigg)^2 
\\=& \ex \sum_i \lambda_i^2 a_i^2 - {\rm Var}\bigg(\sqrt{\sum_i \lambda_i^2 a_i^2}\bigg) \\
=& \frac{\sum_i \lambda_i^2}{d} - \int_t t^2 {\rm d}\pr\left(\Bigg|\sqrt{\sum_i \lambda_i^2 a_i^2} - \ex \sqrt{\sum_i \lambda_i^2 a_i^2}\Bigg| < t \right) \\
=& \frac{\sum_i \lambda_i^2}{d} - \int_t 2t \pr\left(\Bigg|\sqrt{\sum_i \lambda_i^2 a_i^2} - \ex \sqrt{\sum_i \lambda_i^2 a_i^2}\Bigg| > t \right) {\rm d}t \\
\ge& \frac{\sum_i \lambda_i^2 - 4}{d}.
\end{align*}
Taking $\lambda_i = 1$, we prove \eqref{eq:norm}.
Taking square, we prove \eqref{eq:square}.
\end{proof}

Here, we also use $\langle \va, \vb \rangle$ to denote the angle between $\va$ and $\vb$.

\begin{lemma}[Concentration of measure \citep{ledoux2001concentration}] \label{lem:angle}
Assume $\va, \vb \overset{\mathrm{ind.}}{\sim} \mathcal{N}(\bm{0}, \frac{1}{d}\vI_d) \in \mathbb{R}^d$, then for any $t > 0$
$$
\pr\left(|\cos\langle \va, \vb \rangle| > t \right) < 2\me^{-\frac{dt^2}{2}}.
$$
\end{lemma}

\begin{lemma} \label{lem:Bernstein}
$X_1, X_2, \ldots, X_N$ are generated independently from ${\rm Bern}(p)$, 
then for any $t > 0$
$$
\pr\left(\left|\frac{1}{N}\sum_{i}X_i - p\right| > t\right) < \me^{-\frac{t^2N}{p + \frac{1}{3}t}}. 
$$
\end{lemma}

\begin{proof}
According to Bernstein's Inequality, the conclusion is obvious.
\end{proof}

\begin{lemma} \label{lem:eig_property}
For any symmetric matrix $\vM \in \mathbb{R}^{n \times n}$,
$$
\lambda_{i+1} = \min_{S_{n-i}} \max_{\vx \in S_{n-i}} \vx^{\top}\vM\vx,
$$
where $S_{n-i}$ denotes the subspace of $\mathbb{R}^n$ of dimension $n-i$.
\end{lemma}

\begin{proof}
This is a basic property of eigenvalues.
\end{proof}

We define a random subset $\mathcal{S}$ of $[N]$ with average size $\alpha N$ as follows.
For all $i \in [N]$, $i$ belongs to $\mathcal{S}$ with probability $\alpha$ independently from each other.
Then we state an elementary decoupling lemma for double arrays here.

\begin{lemma}[Decoupling \citep{helmers2000decoupling}] \label{lem:dec}
Consider a double array of real numbers $(a_{ij})_{i, j = 1}^{2N}$ such that $a_{ii} = 0$ for all $i$. Then
$$
\sum_{i, j \in [N]} a_{ij} = 4\ex \sum_{i \in \mathcal{S}, j \in \mathcal{S}^c} a_{ij},
$$
where $\mathcal{S}$ is a random subset of $[N]$ with average size $N/2$.
\end{lemma}

\begin{lemma} [Matrix Bernstein: Mgf and Cgf Bound, Lemma 6.6.2 \citep{tropp2015introduction}] \label{lem:Matrix_Bernstein}
Suppose that $\vX$ is a random Hermitian matrix that satisfies
$$
\ex \vX = 0, \lambda_{\max}(\vX) \le L,
$$
then for $0 < \theta < 3/L$
$$
\log \ex \me^{\theta \vX} \preccurlyeq \frac{\theta^2/2}{1-\theta L/3} \ex \vX^2.
$$
\end{lemma}

\vskip 0.2in
\bibliography{mybibfile}

\end{document}